%% file: 0-paper.tex
\documentclass[sigconf]{acmart}

\usepackage{csquotes}
\usepackage{url}
\usepackage{amsmath}
\usepackage{mathtools}
\usepackage{csquotes}
\usepackage{enumitem}
\usepackage{xspace}
\usepackage{algorithmic}
\usepackage{algorithm}
\usepackage{subcaption}
\usepackage{array}
\usepackage{mathtools} 
\usepackage{comment}
\usepackage[group-separator={,}]{siunitx}
\usepackage{bm}
\usepackage{placeins}
\usepackage{amsthm}
\usepackage{xcolor}
\usepackage{soul}
\usepackage{multirow}
\usepackage{makecell}
\usepackage{wrapfig}

\AtBeginDocument{%
  \providecommand\BibTeX{{%
    \normalfont B\kern-0.5em{\scshape i\kern-0.25em b}\kern-0.8em\TeX}}}

\copyrightyear{2020}
\acmYear{2020}
\setcopyright{acmcopyright}\acmConference[CIKM '20]{Proceedings of the 29th ACM International Conference on Information and Knowledge Management}{October 19--23, 2020}{Virtual Event, Ireland}
\acmBooktitle{Proceedings of the 29th ACM International Conference on Information and Knowledge Management (CIKM '20), October 19--23, 2020, Virtual Event, Ireland}
\acmPrice{15.00}
\acmDOI{10.1145/3340531.3411898}
\acmISBN{978-1-4503-6859-9/20/10}

\begin{document}
\fancyhead{}

%%
%% The "title" command has an optional parameter,
%% allowing the author to define a "short title" to be used in page headers.
\title{Carpe Diem, Seize the Samples Uncertain \enquote{at the Moment} \\for Adaptive Batch Selection}

%%
%% The "author" command and its associated commands are used to define
%% the authors and their affiliations.
%% Of note is the shared affiliation of the first two authors, and the
%% "authornote" and "authornotemark" commands
%% used to denote shared contribution to the research.
\author{Hwanjun Song}
\affiliation{%
\institution{Korea Advanced Institute of Science and Technology\\ Daejeon, Korea}
}
\email{songhwanjun@kaist.ac.kr}
\author{Minseok Kim}
\affiliation{%
\institution{Korea Advanced Institute of Science and Technology\\ Daejeon, Korea}
}
\email{minseokkim@kaist.ac.kr}
\author{Sundong Kim}
\affiliation{%
\institution{Institute for Basic Science\\ Daejeon, Korea}
}
\email{sundong@ibs.re.kr}
\author{Jae-Gil Lee}
\affiliation{%
\institution{Korea Advanced Institute of Science and Technology\\ Daejeon, Korea}
}
\email{jaegil@kaist.ac.kr}
\authornote{Jae-Gil Lee is the corresponding author.}

%% By default, the full list of authors will be used in the page
%% headers. Often, this list is too long, and will overlap
%% other information printed in the page headers. This command allows
%% the author to define a more concise list
%% of authors' names for this purpose.

%%
%% The abstract is a short summary of the work to be presented in the
%% article.
\begin{abstract}
The accuracy of deep neural networks is significantly affected by how well mini-batches are constructed during the training step.
In this paper, we propose a novel adaptive batch selection algorithm called \textbf{Recency~Bias} that exploits the uncertain samples predicted inconsistently \emph{in recent iterations}. The historical label predictions of each training sample are used to evaluate its predictive uncertainty within a \emph{sliding window}. Then, the sampling probability for the next mini-batch is assigned to each training sample in proportion to its predictive uncertainty. By taking advantage of this design, \algname{} not only accelerates the training step but also achieves a more accurate network. We demonstrate the superiority of \algname{} by extensive evaluation on two independent tasks. Compared with existing batch selection methods, the results showed that \algname{} reduced the test error by up to $20.97\%$ in a fixed wall-clock training time. At the same time, it improved the training time by up to $59.32\%$ to reach the same test error. 
\end{abstract}

\begin{CCSXML}
<ccs2012>
<concept>
<concept_id>10010147.10010257.10010293.10010294</concept_id>
<concept_desc>Computing methodologies~Neural networks</concept_desc>
<concept_significance>500</concept_significance>
</concept>
<concept>
<concept_id>10010147.10010257.10010258.10010259.10010263</concept_id>
<concept_desc>Computing methodologies~Supervised learning by classification</concept_desc>
<concept_significance>500</concept_significance>
</concept>
</ccs2012>
\end{CCSXML}

\ccsdesc[500]{Computing methodologies~Neural networks}
\ccsdesc[500]{Computing methodologies~Supervised learning by classification}

%% Keywords. The author(s) should pick words that accurately describe
%% the work being presented. Separate the keywords with commas.
\keywords{Batch Selection, Uncertain Sample, Acceleration, Convergence}

\theoremstyle{definition}
\newcommand{\colorcomment}[3]{\xspace{\color{#2} [{#1}]:{#3}}\xspace}
\newcommand{\jaegil}[1]{\colorcomment{Jae-Gil}{blue}{#1}}
\newcommand{\hwanjun}[1]{\colorcomment{HJ}{purple}{#1}}
\newcommand{\sundong}[1]{\colorcomment{Sundong}{orange}{#1}}
\newcommand{\algname}{\emph{Recency Bias}}
\newcommand{\randombatch}{\emph{Random Batch}}
\newcommand{\activebias}{\emph{Active Bias}}
\newcommand{\onlinebatch}{\emph{Online Batch}}
\renewcommand{\algorithmicrequire}{\textsc{Input:}}
\renewcommand{\algorithmicensure}{\textsc{Output:}}
\renewcommand{\algorithmiccomment}[1]{/*~#1~*/}
\newcommand{\INDSTATE}[1][1]{\STATE\hspace{#1\algorithmicindent}}
\DeclarePairedDelimiter{\ceil}{\lceil}{\rceil}

\maketitle

\sloppy

\input{1-introduction}
\input{2-relatedwork}

\input{3-methodology}

\input{4-algorithm}

\input{5-evaluation}

\input{6-conclusion}

\section*{Acknowledgement}
This work was partly supported by Institute of Information \& Communications Technology Planning \& Evaluation\,(IITP) grant funded by the Korea government\,(MSIT) (No.\ 2020-0-00862, DB4DL: High-Usability and Performance In-Memory Distributed DBMS for Deep Learning)
and
the National Research Foundation of Korea\,(NRF) grant funded by the Korea government\,(Ministry of Science and ICT) (No.\ 2017R1E1A1A01075927).  

\bibliography{0-paper}

\begin{appendix}
\input{8-appendix}

\end{appendix}

\end{document}

%% file: 1-introduction.tex
\section{Introduction}
\label{sec:intro}

\begin{figure}[b!]
\vspace*{-0.4cm}
\begin{center}
\includegraphics[width=8.5cm]{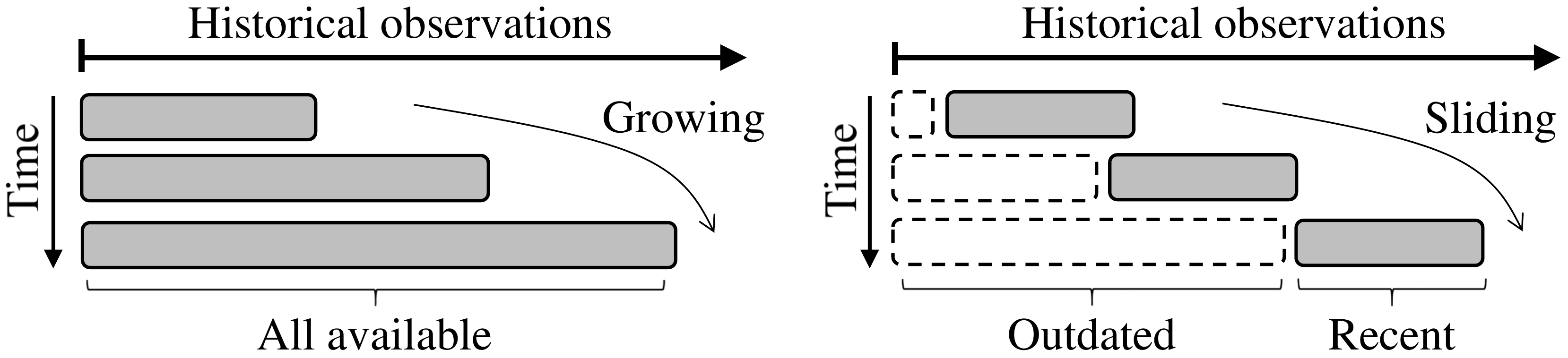}
\end{center}
\vspace*{-0.1cm}
\hspace*{0.2cm} {\small (a) Growing Window.} \hspace*{1.95cm} {\small (b) Sliding Window.}
\vspace*{-0.3cm}
\caption{Two forms of handling the time-series observations.}
\label{fig:time_series}
\end{figure}

Deep neural networks\,(DNNs) have become one of the most popular methods for supervised learning tasks in that traditional machine learning is successfully superseded by recent deep learning in many applications\,\cite{mei2018attentive}. However, in return for this higher popularity and capability, the training of DNNs raises many challenges mainly because of their very high expressive power as well as complex structure. Thus, in this paper, we address an important issue in the training of DNNs, namely \emph{batch selection}.

Stochastic gradient descent\,(SGD) for \emph{randomly} selected mini-batch samples is commonly used to train DNNs. However, many recent studies have pointed out that the performance of DNNs is heavily dependent on how well the mini-batch samples are selected\,\cite{shrivastava2016training,chang2017active,katharopoulos2018not, song2020ada}. In earlier approaches, a sample's \emph{difficulty} is employed to identify proper mini-batch samples, and these approaches achieve a more accurate and robust network\,\cite{han2018co, song2019selfie} or expedite the training convergence of SGD\,\cite{loshchilov2015online}. However, the two opposing difficulty-based strategies, i.e., preferring \emph{easy} samples\,\cite{kumar2010self,song2019selfie} versus \emph{hard} samples\,\cite{loshchilov2015online,shrivastava2016training}, work well in different situations.
The former results in a network robust to outliers and noisy labels, but slows down the training process by causing small gradients\,\cite{meng2015objective, song2019prestopping}; the latter accelerates the training process, but leads to poor generalization on test data by exacerbating the overfitting\,\cite{loshchilov2015online}.
Thus, for practical reasons to cover more diverse situations, recent approaches begin to exploit a sample's \emph{uncertainty} that indicates the consistency of previous predictions\,\cite{chang2017active, song2019selfie, song2020learning}. %\looseness=-1

\begin{figure*}[t!]
\begin{center}
\includegraphics[width=16.0cm]{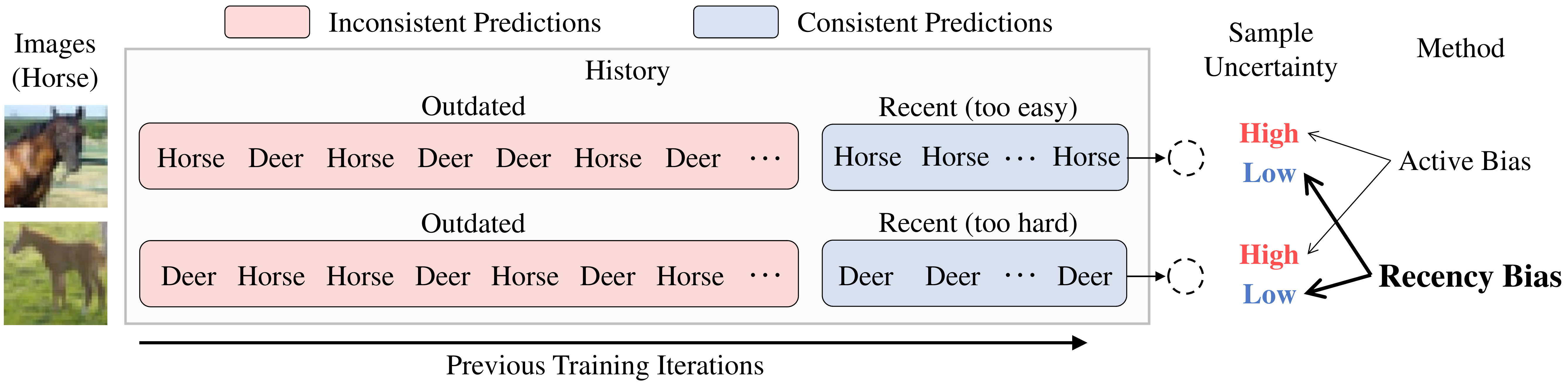}
\end{center}
\vspace*{-0.3cm}
\caption{The difference in sample uncertainty estimated by \activebias{} and \algname{}.}
\label{fig:differenc_btw_two_methods}
\vspace*{-0.3cm}
\end{figure*}

An important question here is how to evaluate a sample's uncertainty based on its historical predictions during the training process. Intuitively, because a series of historical predictions can be seen as a series of data indexed in chronological order, the uncertainty can be measured based on \emph{two} forms of handling time-series observations: \emph{(i)} a \emph{growing window}\,(Figure \ref{fig:time_series}(a)) that consistently increases the size of a window to use all available observations and \emph{(ii)} a \emph{sliding window}\,(Figure \ref{fig:time_series}(b)) that maintains a window of a fixed size on the most recent observations by deleting outdated ones. While the state-of-the-art algorithm, \activebias{}\,\cite{chang2017active}, adopts the growing window, we propose to use the sliding window in this paper. \looseness=-1

In more detail, \activebias{} recognizes uncertain samples based on the inconsistency of the predictions in the \emph{entire} history of past SGD iterations. Then, it emphasizes such uncertain samples by choosing them with high probability for the next mini-batch. However, according to our experiments presented in Section \ref{sec:eval_classification}, such uncertain samples \emph{slowed down} the convergence speed of training, though they ultimately reduced the generalization error. This weakness is attributed to the inherent limitation of the growing window, where older observations could be too outdated\,\cite{torgo2011data}. In other words, the outdated predictions no longer represent a network's current behavior. As illustrated in Figure \ref{fig:differenc_btw_two_methods}, when the label predictions of two samples were inconsistent for a long time, \activebias{} invariably regards them as highly uncertain, although their recent label predictions become consistent along with the network's training progress. This aspect evidently entails the risk of emphasizing uninformative samples---too easy or too hard---at the current moment, thereby slowing down the convergence speed of training.

%\,(\emph{\underline{L}ooking \underline{U}ncertain \underline{S}amples for DNNs \underline{T}raining})
Therefore, we propose a simple but effective batch selection method, called \textbf{Recency~Bias}, that takes advantage of the \emph{sliding window} to evaluate the uncertainty in \emph{fresher} observations. As opposed to \activebias{}, \algname{} excludes the outdated predictions by managing a sliding window of a fixed size and picks up the samples predicted inconsistently within the sliding window. Thus, as shown in Figure \ref{fig:differenc_btw_two_methods}, the two samples uninformative at the moment are no longer selected by \algname{} simply because their recent predictions are consistent. Consequently, since informative samples are effectively selected throughout the training process, this strategy not only accelerates the training speed but also leads to a more accurate network.

%To validate the superiority of \algname{}, two popular convolutional neural networks\,(CNNs)\footnote{The idea is applicable to the DNNs other than CNNs, and we leave this extension as the future work.} were trained for two independent tasks: image classification and fine tuning.
%We compared \algname{} with not only random batch selection\,(baseline) but also two state-of-the-art batch selection strategies.
%Compared with three batch selection strategies, \algname{} provided a relative reduction of test error by $1.81\%$--$20.5\%$ in a fixed wall-clock training time. At the same time, it significantly reduced the execution time by $24.6\%$--$59.3\%$ to reach the same test error.

To validate the superiority of \algname{}, two popular convolutional neural networks\,(CNNs)
%\footnote{The idea is applicable to the DNNs other than CNNs, and we leave this extension as the future work.} 
were trained for two independent tasks: image classification and fine tuning. We compared \algname{} with not only the random batch selection\,(baseline) but also \emph{Online Batch} and \emph{Active Bias}, which are the two state-of-the-art adaptive batch selection methods. Our extensive experiments empirically confirmed the following promising results:
\begin{itemize}[leftmargin=9pt]
\item \textbf{Uncertainty Estimation}: \algname{} benefited from using the sliding window in estimating the uncertainty. It evaluated truly uncertain samples during the entire training process, while the growing window approach\,(\emph{Active Bias}) misclassified many easy samples as uncertain at a late stage of training. \looseness=-1
\vspace*{0.05cm}
\item \textbf{Accuracy and Efficiency}: \algname{} provided a relative reduction in test error by up to $20.97\%$ in a fixed wall-clock training time compared with three batch selection strategies. At the same time, it significantly reduced the execution time by up to $59.32\%$ to reach the same test error.
\vspace*{0.05cm}
\item \textbf{Practicality}: The performance dominance of \algname{} was consistent on two learning tasks based on various benchmark datasets. The impact on both convergence speed and generalization capability have a great potential to improve many other deep learning tasks.
\end{itemize}

In the rest of the paper, Section \ref{sec:related_work} reviews related work, Section \ref{sec:recency_bias} proposes the \algname{} algorithm, Section \ref{sec:evaluation} presents the experiments, and Section \ref{sec:conclusion} concludes the paper.

%% file: 2-relatedwork.tex
\vspace*{-0.1cm}
\section{Related Work}
\label{sec:related_work}

SGD is a popular optimization method, which has been extensively studied in the machine learning community\,\cite{johnson2013accelerating, mahdavi2012stochastic, shamir2013stochastic, liu2019loopless}. Typically, at each training iteration, more than one training samples called a \emph{mini-batch} is constructed from the training dataset and then employed to update the model parameter such that this parameter minimizes the empirical risk on the mini-batch.
Let $\mathcal{D}=\{(x_i, y_i)\}_{i=1}^{N}$ be the entire training dataset composed of a sample $x_i$ with its true label $y_i$, where $N$ is the total number of training samples. Then, a straightforward strategy to construct a mini-batch $\mathcal{B}=\{(x_i, y_i)\}_{i=1}^{b}$ is to select $b$ samples \emph{uniformly at random}\,(i.e., $P(x_i|\mathcal{D})=1/N$) from the training dataset $\mathcal{D}$. 

Because not all samples have an equal impact on training, many research efforts have been devoted to develop advanced \emph{sampling schemes}. Bengio et al.~\shortcite{bengio2009curriculum} first took easy samples and then gradually increased the difficulty of samples using heuristic rules. Kumar et al.~\shortcite{kumar2010self} determined the easiness of the samples using their prediction errors. Recently, Tsvetkov et al.~\shortcite{tsvetkov2016learning} used Bayesian optimization to learn an optimal curriculum for training dense, distributed word representations. Sachan and Xing~\shortcite{sachan2016easy} emphasized that the right curriculum must introduce a small number of the samples dissimilar to those previously seen. Fan et al.~\shortcite{fan2017neural} proposed a neural data filter based on reinforcement learning to select training samples adaptively. However, it is common for deep learning to emphasize \emph{hard} samples because of the plethora of easy ones\,\cite{katharopoulos2018not}.

In light of this, Loshchilov and Hutter~\shortcite{loshchilov2015online} proposed a \emph{difficulty}-based sampling scheme, called \onlinebatch{}, that uses the rank of the loss computed from previous epochs. \onlinebatch{} sorts the previously computed losses of samples in descending order and exponentially decays the sampling probability of a sample according to its rank $r$. Then, the $r$-th ranked sample $x(r)$ is selected with the probability dropping by a factor of $\exp{\big(\log(s_e)/N\big)}$, where $s_e$ is the \emph{selection pressure} parameter that affects the probability gap between the most and the least important samples. When normalized to sum to $1.0$, the probability $P(x(r)|\mathcal{D};s_e)$ is defined by
\begin{equation}
\label{eq:online_batch_p}
P(x(r)|\mathcal{D};s_e) = \frac{1/\exp{\big(\log(s_e)/N\big)^r}}{\sum_{j=1}^{N}1/\exp{\big(\log(s_e)/N\big)^j}}.
\end{equation}
However, it has been reported that \onlinebatch{} accelerates the convergence of training loss but deteriorates the generalization on test data because of the overfitting to hard training samples\,\cite{loshchilov2015online}. Thus, as confirmed in our experiment in Section \ref{sec:eval_classification}, it only works well for easy datasets, where hard training samples rarely exist.

Most close to our work, Chang et al.~\shortcite{chang2017active} devised an \emph{uncertainty}-based sampling scheme, called \activebias{}, that chooses uncertain samples with high probability for the next batch. \activebias{} maintains the history $\mathcal{H}_{i}^{t-1}$ that stores \emph{all} $h(y_i|x_i)$ before the current iteration $t$\,(i.e., growing window), where $h(y_i|x_i)$ is the softmax probability of a given sample $x_i$ for its true label $y_i$. Then, it measures the uncertainty of the sample $x_i$ by computing the variance over \emph{all} $h(y_i|x_i)$ in $\mathcal{H}_{i}^{t-1}$ and draws the next mini-batch samples based on the normalized probability $P(x_i|\mathcal{D}, \mathcal{H}_{i}^{t-1}; \epsilon)$ defined by 
\begin{equation}
\label{eq:active_bias_p}
\begin{gathered}
P(x_i|\mathcal{D}, \mathcal{H}_{i}^{t-1}; \epsilon) = \frac{\hat{std}(\mathcal{H}_{i}^{t-1})+ \epsilon}{\sum_{j=1}^{N}\big(\hat{std}(\mathcal{H}_{j}^{t-1})+\epsilon\big)},\\
\text{where\,\,\,}\hat{std}(\mathcal{H}_{i}^{t-1})=\sqrt{var(\mathcal{H}_{i}^{t-1}) + \frac{var(\mathcal{H}_{i}^{t-1})^{2}}{|\mathcal{H}_{i}^{t-1}|-1}}
\end{gathered}
\end{equation}
and $\epsilon$ is a smoothness constant to prevent the low variance samples from never being selected again.
As mentioned earlier in Introduction, \activebias{} slows down the training process compared to random batch selection because the oldest part in the history $\mathcal{H}_{i}^{t-1}$ no longer represents the current behavior of the network.

For the completeness of the survey, we include the recent studies on submodular batch selection. Joseph et al.~\shortcite{joseph2019submodular} and Wang et al.~\shortcite{wang2019fixing} designed their own submodular objectives that cover diverse aspects, such as sample redundancy and sample representativeness, for more effective batch selection. Differently from their work, we explore {the issue of truly uncertain samples} in an orthogonal perspective. Our uncertainty measure can be easily injected into their submodular optimization framework as a measure of sample informativeness. \looseness=-1

%In Section \ref{sec:evaluation}, we will confirm that \algname{} outperforms \onlinebatch{} and \activebias{}, which are regarded as two state-of-the-art adaptive batch selection methods.

%% file: 3-methodology.tex
\section{Batch Selection via \algname{}}
\label{sec:recency_bias}

\subsection{Problem Setting and Overview}
\label{sec:overview}

In the standard training of DNNs by  SGD, the parameter $\theta$ of the network is updated according to the descent direction of the expected empirical risk on the given mini-batch $\mathcal{B}$, 
\begin{equation}
\theta_{t+1} = \theta_{t} - \eta \nabla\big( \frac{1}{|\mathcal{B}|}\! \sum_{x_i \in \mathcal{B}} f_{x_i}(\theta_t) \big),
\end{equation}
where $\eta$ is the given learning rate and $f_{x_i}(\theta_t)$ is the empirical risk of the sample $x_i$ for the network parameterized by $\theta_t$. 

In this study, we modify the update rule to highlight uncertain samples during the training process. First, \algname{} manages a sliding window of a fixed size, called the \emph{label history} $\mathcal{H}$, which stores recent label predictions of all training samples. Second, \algname{} evaluates the sample uncertainty based on the label history and then builds a sampling distribution $P(x|\mathcal{D}, \mathcal{H} ;s_e)$ such that uncertain samples exhibit higher probabilities, where $s_e$ is the selection pressure parameter. Let $\mathcal{U}$ be the uncertain mini-batch samples drawn according to the sampling distribution built at the current moment. Then, the modified update rule is  % \looseness=-1
\begin{equation}
\begin{gathered}
%P(x|\mathcal{D};s_e) \leftarrow Sampling Distribution(\mathcal{H})\\
\theta_{t+1} = \theta_{t} - \eta \nabla\big( \frac{1}{|\mathcal{U}|}\! \sum_{x_i \in \mathcal{U}} f_{x_i}(\theta_t) \big), \text{where\,\,} \mathcal{U} \!\sim \!P(x|\mathcal{D}, \mathcal{H}; s_e).
\end{gathered}
\label{eq:uncertain_update}
\vspace*{0.00cm}
\end{equation}
Please note that \algname{} iteratively refines the sampling distribution by using the latest label history \emph{once per epoch} because an epoch is a widely-used learning cycle to measure the model changes\,\cite{loshchilov2015online, goodfellow2016deep}.
To update the network by Eq.\ \eqref{eq:uncertain_update}, the main challenge is how to evaluate the uncertainty of training samples based on their label histories as well as how to assign the sampling probability for them to construct the next mini-batch. The two main challenges are detailed in the following section.% \looseness=-1

\subsection{Batch Selection Methodology}
\label{sec:methodology}

\subsubsection{Criterion of an Uncertain Sample}

Intuitively, samples are uncertain if their \emph{recent} label predictions are highly inconsistent because they are neither too easy nor too hard at the moment.
Thus, we adopt the \emph{predictive uncertainty}\,\cite{song2019selfie} in Definition\,\ref{def:predictive_uncertainty} that uses the information entropy\,\cite{chandler1987introduction} to measure the inconsistency of recent label predictions. Here, the sample with high predictive uncertainty is regarded as uncertain and thus should be selected with high probability for the next mini-batch. \looseness=-1

\begin{definition}{\bf (Predictive Uncertainty)}
\label{def:predictive_uncertainty}
Given a neural network $\Phi$ parameterized by $\theta$, let $\hat{y}_{t}=\Phi(x_i; \theta_{t})$ be the predicted label of a sample $x_i$ at time $t$ and $\mathcal{H}_{i}(q)=$ $\{\hat{y}_{t_{1}}, \hat{y}_{t_{2}},\dots,$ $\hat{y}_{t_{q}}\}$ be the label history of the sample $x_i$ that stores the predicted labels at the previous $q$ times. The label history $\mathcal{H}_{i}(q)$ corresponds to the \emph{sliding window} of size $q$ to compute the uncertainty of the sample $x_i$.
Next, based on $\mathcal{H}_{i}(q)$, the probability of the $j$-th label\,($j\in\{1,2,\ldots,k\}$) estimated as the label of the sample $x_i$ is formulated by
\vspace*{0.2cm}
\begin{equation}
\label{eq:label_prob}
p(y_i=j|x_i) = \frac{\sum_{\hat{y} \in \mathcal{H}_{i}(q)}[\hat{y}=j]}{|\mathcal{H}_{i}(q)|},
\vspace*{0.1cm}
\end{equation}
where $[\cdot]$ is the Iverson bracket\footnote{The Iverson bracket $[p]$ returns $1$ if $p$ is true; $0$ otherwise.}.
Then, to quantify the uncertainty of the sample $x_i$, the empirical entropy is used to define the \emph{predictive uncertainty} $U(x_i)$ by
\vspace*{0.1cm}
\begin{equation}
\label{eq:predictive_uncertainty}
\begin{gathered}
U(x_i)=\frac{1}{\log(k)} entropy\big(p(y_i|x_i)\big),\\
\text{where\,\,}entropy\big(p(y_i|x_i)\big) =-\!\sum_{j=1}^{k}p(y_i\!=\!j|x_i)\log{p(y_i\!=\!j|x_i)}.
\end{gathered}
\vspace*{0.1cm}
\end{equation}
\end{definition} 
For $k$ classes, to ensure $0 \leq U(x_i) \leq 1$, the empirical entropy is divided by $\log(k)$, which is the maximum value obtained when $\forall_{j\in\{1,\dots,k\}} p(y_i=j|x_i)=1/k$.
%Because the uncertainty is bounded, we add the standardization term $\delta$ to normalize the value to $[0,1]$. For $k$ classes, $\delta$ is the maximum entropy when $\forall_{j} %p(j|x_i;q)=1/k$. 

\subsubsection{Sampling Probability for Mini-batch Construction}
\label{sec:sample_prob}

To construct next mini-batch samples, we assign the sampling probability according to the predictive uncertainty in Definition \ref{def:predictive_uncertainty}.
Motivated by Loshchilov and Hutter~\shortcite{loshchilov2015online}, the sampling probability of a given sample $x_i$ is exponentially decayed with its predictive uncertainty $U(x_i)$. In detail, we adopt the quantization method\,\cite{chen2001quantization} and use the quantization index to decay the sampling probability. 
The index is obtained using the simple quantizer $Q(z)$ defined by 
\vspace*{0.1cm}
\begin{equation}
Q\big(z\big)= \ceil{\big(1-z\big) /\Delta},~~~~ 0\leq z \leq 1,
\label{eq:quantization_f}
\vspace*{0.1cm}
\end{equation}
where $\Delta$ is the quantization step size. Compared with the rank-based index\,\cite{loshchilov2015online}, the quantization index is known to well reflect the difference in actual values\,\cite{widrow1996statistical}.
Comparing two different epochs in Figure \ref{fig:index_analysis}, although the predictive uncertainty of the samples are changed between the two epochs, the rank-based approach ignores this change by assigning a fixed index for the same rank, whereas our quantization approach assigns different indexes every moment by reflecting the change in actual values.

\begin{figure}[t!]
\begin{center}
\includegraphics[width=8.4cm]{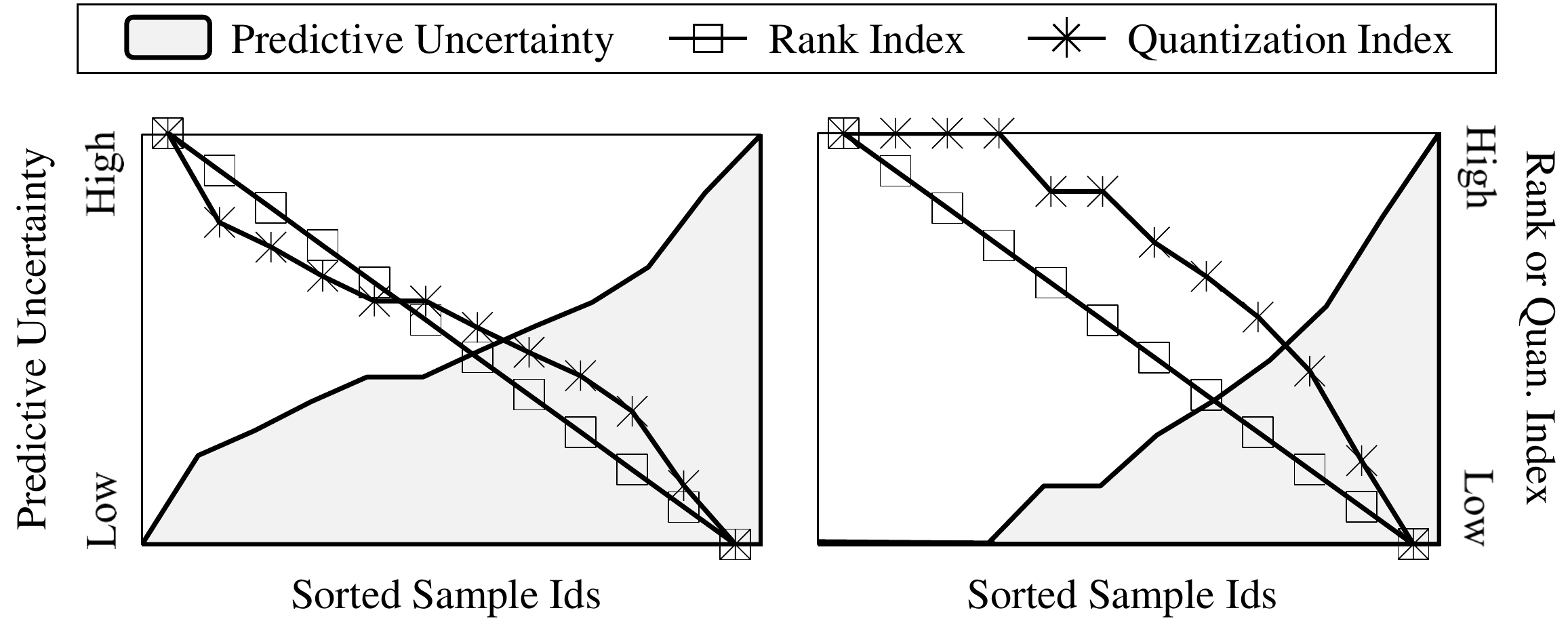}
\end{center}
\vspace*{-0.05cm}
\hspace*{0.1cm} {\small (a) Early Stage\,(30\%).}  \hspace*{1.1cm} {\small (b) Later Stage\,(70\%).}  
\vspace*{-0.25cm}
\caption{Comparison of the index computed by the rank-based and quantization-based approaches: (a) and (b) show the predictive uncertainty of training samples in CIFAR-100 with their indexes at the 30\% and 70\% of total training epochs, respectively.}
\label{fig:index_analysis}
\vspace*{-0.2cm}
\end{figure}

In Eq.~\eqref{eq:quantization_f}, we set $\Delta$ to be $1/N$ such that the index is bounded to $N$\,(the total number of samples).
Then, the sampling probability $P(x_i|\mathcal{D}, \mathcal{H}; s_e)$ is defined by 
\begin{equation}
\label{eq:exp_decay_q_index}
P(x_i|\mathcal{D}, \mathcal{H}; s_e) = \frac{1/\exp{\big(\log(s_e)/N\big)^{Q(U(x_i))}}}{\sum_{j=1}^{N}1/\exp{\big(\log(s_e)/N\big)^{Q(U(x_j))}}}.
\end{equation}
The higher the predictive uncertainty, the smaller the quantization index. Therefore, a higher sampling probability is assigned for uncertain samples by Eq.~\eqref{eq:exp_decay_q_index}.

Meanwhile, it is known that using only some part of training data exacerbates the overfitting problem at a late stage of training\,\cite{loshchilov2015online,zhou2018minimax}. 
Thus, to alleviate the problem, we include more training samples as the training progresses by exponentially decaying the selection pressure $s_e$ by
\begin{equation}
\label{eq:exp_decay_s_e}
s_e = s_{e_{0}}\Big({\rm \text{exp}}\big( \log{(1/s_{e_{0}})}/(e_{end}-e_{0}) \big)\Big)^{e-e_{0}}.
\end{equation}
At each epoch $e$ from $e_{0}$ to $e_{end}$, the selection pressure $s_e$ exponentially decreases from $s_{e_{0}}$ to $1$. 
Because this technique gradually reduces the sampling probability gap between the most and the least uncertain samples, more diverse samples are selected for the next mini-batch at a later epoch. When the selection pressure $s_e$ becomes $1$, the mini-batch samples are randomly chosen from the entire dataset. 
The ablation study on the effect of decaying the selection pressure is presented in Section \ref{sec:abalation}. 

\vspace*{0.15cm}
\subsection{Convergence Guarantee}
\label{sec:convergence}
\vspace*{0.05cm}

All adaptive batch selection strategies are widely known to follow the typical convergence guarantee of the vanilla SGD\,\cite{robbins1951stochastic}, as long as their sampling distributions are \emph{strictly positive} and their gradient estimates are \emph{unbiased}\,\cite{gopal2016adaptive, zhao2015stochastic}. Hence, we provide the theoretical evidence that \algname{} satisfies the two aforementioned conditions for the convergence guarantee.

\vspace*{+0.1cm}
\begin{lemma}
Let $P(x|\mathcal{D},\mathcal{H};s_e)$ be the sampling distribution of Recency Bias. Then, it is a strictly positive distribution.
\begin{proof}
By Eq.~\eqref{eq:predictive_uncertainty} and Eq.~\eqref{eq:quantization_f}, the index $Q(U(x))$ is bounded by
\begin{equation}
0 \leq Q(U(x)) \leq \ceil{1/\Delta}.
\end{equation}
Then, the lower bound of $P(x|\mathcal{D},\mathcal{H};s_e)$ is formulated by
\begin{equation}
0 < \frac{1}{\sum_{j=1}^{N}1/\exp{\big(\log(s_e)/N\big)^{Q(U(x_j))}}} \leq P(x|\mathcal{D},\mathcal{H};s_e).
\end{equation}
Thus, the sampling distribution of \emph{Recency Bias} is strictly positive because $P(x_i|\mathcal{D},\mathcal{H};s_e) > 0$ for all $x_i \in \mathcal{D}$.
\end{proof}
\label{lemma:positive}
\end{lemma}

\begin{lemma}
Let $\tilde{G}$ be the gradient estimate of Recency Bias. Then, $\tilde{G}$ is an unbiased estimator.
\begin{proof}
Let $f_{x_i}(\theta)$ be the empirical risk of the sample $x_i$ for the given network parameterized by $\theta$. Then, because the mini-batch samples in \algname{} are drawn according to $P(x|\mathcal{D}, \mathcal{H};s_e)$, the gradient estimate $\tilde{G}$ is equivalent to that of the weighted empirical risk on the mini-batch $\mathcal{B}$ randomly drawn from $\mathcal{D}$\,\cite{gopal2016adaptive}, \looseness=-1
\begin{equation}
\tilde{G} = \nabla \sum_{x_i \in \mathcal{B}}w(x_i)f_{x_i}(\theta), 
~~~~w(x_i) = \frac{P(x_i|\mathcal{D}, \mathcal{H};s_e)}{\sum_{x_j \in \mathcal{B}}P(x_j|\mathcal{D}, \mathcal{H};s_e)}.
\end{equation}
Here, $w(x_i)$ is a constant value by Eq.~\eqref{eq:exp_decay_q_index}, and therefore the estimate $\tilde{G}$ is rewritten as
\begin{equation}
\tilde{G} = \sum_{x_i \in \mathcal{B}}w(x_i)\nabla f_{x_i}(\theta), \text{\,\,\,where\!}\sum_{x_i \in \mathcal{B}}w(x_i)=1.
\end{equation}
That is, $\tilde{G}$ is the weighted combination of the gradient estimate  $\nabla f_{x_i}(\theta)$ of each training sample $x_i$, which is an unbiased estimate of the true gradient,
\begin{equation}
\mathbb{E}[\nabla f_{x_i}(\theta)] = \sum_{x_i \in \mathcal{D}}\frac{1}{|\mathcal{D}|}\nabla f_{x_i}(\theta) = \frac{1}{|\mathcal{D}|} \nabla \sum_{x_i \in \mathcal{D}}f_{x_i}(\theta).
\end{equation}
Thus, the gradient estimate of \emph{Recency Bias} naturally becomes an unbiased estimator due to the linearity of expectation.
\end{proof}
\label{lemma:unbias}
\end{lemma}

%\begin{definition}
%A function $f:\mathbb{R}^{d} \rightarrow \mathbb{R}$ is said to be \emph{L-smooth convex} if $f$ is a convex function satisfying
%\begin{equation}
%f(b) \leq f(a) + \Delta f(a)^{T}(b-a) + \frac{L}{2}||a-b||_{2}^{2}, \text{\,\,\,} \forall a, b \in \mathbb{R}^{d}. \qed
%\end{equation}
%\label{label:convex_smooth}
%\end{definition}

\begin{figure*}[ht!]
\begin{center}
\includegraphics[width=15.0cm]{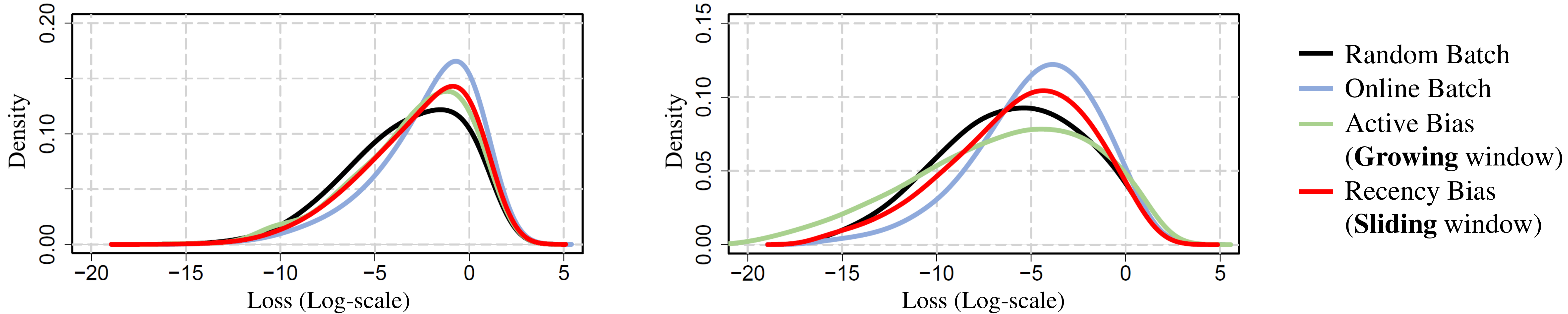}
\end{center}
\vspace*{-0.03cm}
\hspace*{0.7cm} {\small (a) Early Stage\,(30\%).} \hspace*{3.73cm} {\small (b) Late Stage\,(70\%). \hspace*{3.1cm}}
\vspace*{-0.24cm}
\caption{The loss distribution of mini-batch samples selected by four batch selection methods: (a) and (b) show the loss distribution at the $30\%$ and $70\%$ of total training epochs, respectively.}
\label{fig:loss_distribution}
\vspace*{-0.15cm}
\end{figure*}

\begin{theorem}
The SGD by Recency Bias is guaranteed to converge.
More specifically, supposing that $\theta_{*} = \text{\rm argmin}_{\theta} f_{\mathcal{D}}(\theta)$, where $f_{\mathcal{D}}(\theta)$ is the empirical risk of the model parameterized by $\theta$ for all $x_i \in \mathcal{D}$, the empirical risk of Recency Bias at iteration $t$ satisfies the convergence rate of $O(1/\sqrt{t})$\,\cite{shamir2013stochastic},\looseness=-1
\begin{equation}
\mathbb{E}[f_{\mathcal{D}}(\theta_{t})]-f_{\mathcal{D}}(\theta_{*}) = O(1/\sqrt{t}).
\end{equation}
\end{theorem}
\begin{proof}
Done by Lemma \ref{lemma:positive} and Lemma \ref{lemma:unbias}.
\end{proof}
 
%Please note that \algname{} is guaranteed to converge as long as the gradient estimate is unbiased because its sampling distribution $P(x|\mathcal{D}; s_e)$ is a strictly positive distribution\,\cite{gopal2016adaptive}. We do not repeat formal convergence analysis because of the lack of space.

%% file: 4-algorithm.tex
\subsection{Algorithm Pseudocode}
\label{sec:algorithm}

\algsetup{linenosize=\small}\newlength{\oldtextfloatsep}\setlength{\oldtextfloatsep}{\textfloatsep}
\setlength{\textfloatsep}{11.40pt}% Remove \textfloatsep
\begin{algorithm}[!t]
\caption{\algname{} Algorithm}
\label{alg:proposed_algorithm}
\begin{algorithmic}[1]
\REQUIRE { $\mathcal{D}$: data, $epochs$, $b$: batch size, $q$: window size, $s_{e_{0}}$: initial selection pressure, $\gamma$: warm-up}
\ENSURE { $\theta_t$: model parameter}
\STATE {$t \leftarrow 1;$}~~{$\theta_{t} \leftarrow \text{Initialize the model parameter};$}
\STATE {$\mathcal{H} \leftarrow  \text{Initialize the label history of size } q;$}
\STATE {{\bf for} $i=1$ {\bf to} $epochs$ {\bf do}}
\INDSTATE[1] {\COMMENT{{\bf Sampling Probability Derivation}}}
\INDSTATE[1] {{\bf if} $i > \gamma$ {\bf then} }
\INDSTATE[2] {$s_e \leftarrow$ Decay\_Pressure($s_{e_{0}}$, $i$);} ~~{\COMMENT{Decaying $s_e$ by Eq.\,(\ref{eq:exp_decay_s_e})}}\!
\INDSTATE[2] {\COMMENT{Update the index and sampling probability in a batch}}\!\!\!
\INDSTATE[2] {{\bf for} $m=1$ {\bf to} $N$ {\bf do} }
\INDSTATE[3] {$q\_dict[x_m]=Q\big(U(x_m)\big);$ ~~\COMMENT{By Eq.\,(\ref{eq:quantization_f})}} 
\INDSTATE[2] {$p\_table \leftarrow$ Compute\_Prob($q\_dict$, $s_e$);~~\COMMENT{By Eq.\,(\ref{eq:exp_decay_q_index})}}
\INDSTATE[1] {\COMMENT{\bf Network Training}}
\INDSTATE[1] {{\bf for} $j=1$ {\bf to} $N/b$ {\bf do}}
\INDSTATE[2] {{\bf if} $i \leq \gamma$ {\bf then} ~~\COMMENT{Warm-up}}
\INDSTATE[3] {$\{(x_l,y_l)\}_{l=1}^{b} \leftarrow $ Randomly select mini-batch samples;}\!\!
\INDSTATE[2] {{\bf else}~~\COMMENT{Adaptive batch selection}}
\INDSTATE[3] {$\{(x_l,y_l)\}_{l=1}^{b} \leftarrow$ Select mini-batch samples by $p\_table$;
\INDSTATE[2] $losses, labels \leftarrow$ Inference($\{(x_l,y_l)\}_{l=1}^{b}$,$\theta_{t}$);~~\COMMENT{Forward}}\!\!
\INDSTATE[2] {$\theta_{t+1} \leftarrow$ SGD($losses$, $\theta_{t}$);~~\COMMENT{Backward}}
\INDSTATE[2] {\COMMENT{History update using a sliding window}}
%~~\COMMENT{By Definition\,\ref{def:predictive_uncertainty}
\INDSTATE[2] {$\mathcal{H} \leftarrow$ Update\_Label\_History($\mathcal{H}, labels$);}
\INDSTATE[2] {$t \leftarrow t+1;$}
\STATE {\bf return} $\theta_{t}$;
\end{algorithmic}
\end{algorithm} 

Algorithm \ref{alg:proposed_algorithm} describes the overall procedure of \algname{}. The algorithm requires a warm-up period of $\gamma$ epochs because the quantization index for each sample is not confirmed yet. During the warm-up period, which should be at least $q$ epochs\,($\gamma \geq q$) to obtain the label history of size $q$, randomly selected mini-batch samples are used for the network update\,(Lines 13--14). 
After the warm-up period, the algorithm decays the selection pressure $s_e$ and updates not only the quantization index but also the sampling probability in a batch at the beginning of each epoch\,(Lines 5--10). Subsequently, the uncertain samples are selected for the next mini-batch according to the updated sampling probability\,(Lines 15--16), and then the label history is updated along with the network update\,(Lines 17--21). The entire procedure repeats for a given number of epochs. \looseness=-1
\newline \indent
Overall, the key technical novelty of \algname{} is to incorporate the notion of a \emph{sliding window} rather than a growing window into adaptive batch selection, thereby improving both training speed and generalization error. 

\vspace*{0.1cm}
\noindent \textbf{\underline{Time Complexity}}:
The main ``additional'' cost of \algname{} is the derivation of the sampling probability for each sample\,(Lines 5--10). Because only simple mathematical operations are needed per sample, its time complexity is linear to the number of samples\,(i.e., $O(N)$), which is negligible compared with that of the forward and backward steps of a complex network\,(Lines 17--18). Therefore, we contend that \algname{} does \emph{not} add the complexity of an underlying optimization algorithm.

%% file: 5-evaluation.tex
\newcolumntype{L}[1]{>{\raggedright\let\newline\\\arraybackslash\hspace{0pt}}m{#1}}
\newcolumntype{X}[1]{>{\centering\let\newline\\\arraybackslash\hspace{0pt}}p{#1}}

\begin{table}[t!]
%\small
\caption{Summary of datasets used for the two independent tasks in Sections \ref{sec:eval_classification} and \ref{sec:eval_fine_tuning}.}
\vspace*{-0.2cm}
\begin{center}
\begin{tabular}{|L{1.85cm} ||X{1.17cm} |X{1.17cm}|X{1.17cm}|X{1.17cm}|}\hline
Dataset & \!\!\# Training\!\! & \!\!\# Testing\!\! & \!\!\# Classes\!\! & \!\!Resolution\!\! \\ \hline\hline
\!MNIST\,\cite{lecun1998mnist}\!\! & 60,000 & 10,000 &  10 & 28$\times$28 \\\hline
\!CIFAR-10\,\cite{krizhevsky2014cifar}\!\! & 50,000 & 10,000 & 10 & 32$\times$32 \\\hline
\!CIFAR-100\,\cite{krizhevsky2014cifar}\!\! & 50,000 & 10,000 & 100 & 32$\times$32 \\\hline
\!MIT-67\,\cite{quattoni2009recognizing}\!\! & 5,360 & 1,340 & 67 & 256$\times$256 \\\hline
\!FOOD-100\,\cite{kawano14b}\!\! & 7,000 & 3,000 & 100 & 256$\times$256 \\\hline
\end{tabular}
\end{center}
\label{table:datasets}
\vspace*{+0.1cm}
\end{table}

{
\newcommand{\spacebetweenfigs}{-0.45cm}
\newcommand{\spacebeforecaption}{0.2cm}

\begin{figure*}[t]
% make more space by increasing margin <- may not need this
\advance\leftskip-0.3cm % increase left margin temporarily
% \advance\rightskip-1.0cm % increase right margin temporarily
\begin{subfigure}[t!]{1.0\textwidth}
\begin{center}
\includegraphics[width=0.7\textwidth]{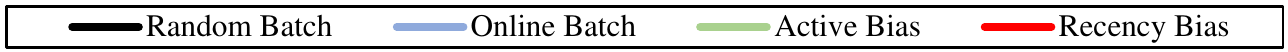}
\vspace*{0.1cm}
\end{center}
\end{subfigure}
%%%%%%%%%%%%%%%%%% 1st subfigure START %%%%%%%%%%%%%%%%%%
\begin{center}
\begin{subfigure}[t!]{0.34\textwidth}
\includegraphics[width=1.0\textwidth]{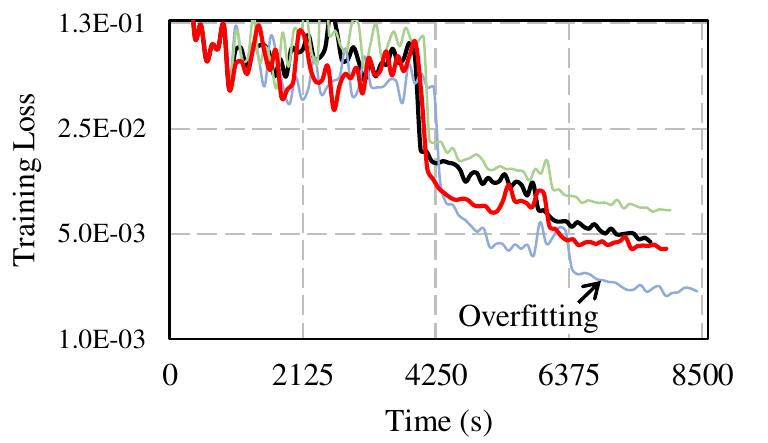}
\end{subfigure}
\hspace*{\spacebetweenfigs}
\begin{subfigure}[t!]{0.34\textwidth}
\includegraphics[width=1.0\textwidth]{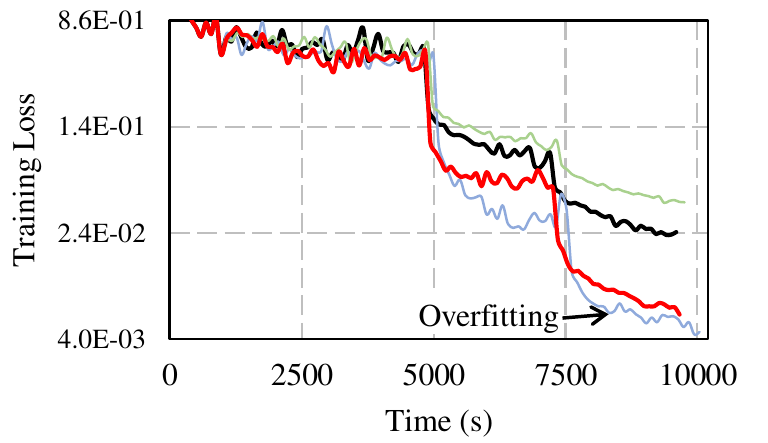}
\end{subfigure}
\hspace*{\spacebetweenfigs}
\begin{subfigure}[t!]{0.34\textwidth}
\includegraphics[width=1.0\textwidth]{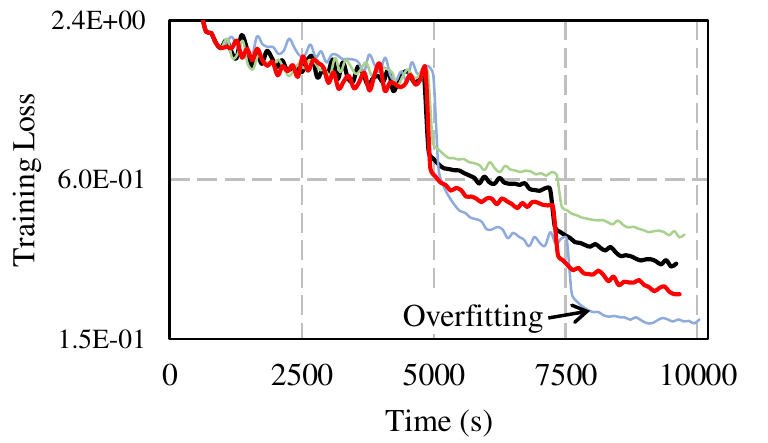}
\end{subfigure}
\end{center}
\hspace*{1.3cm} {\small (a) MNIST Training Loss.} \hspace*{2.5cm} {\small (b) CIFAR-10 Training Loss.} \hspace*{2.4cm} {\small (c) CIFAR-100 Training Loss.}
%%%%%%%%%%%%%%%%%% 2nd subfigure START %%%%%%%%%%%%%%%%%%
\vspace*{0.1cm}
\begin{center}
\begin{subfigure}[t!]{0.34\textwidth}
\includegraphics[width=1.0\textwidth]{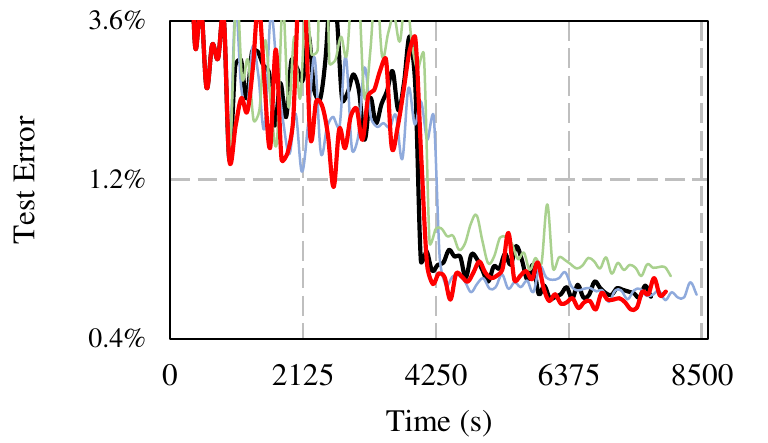}
\end{subfigure}
\hspace*{\spacebetweenfigs}
\begin{subfigure}[t!]{0.34\textwidth}
\includegraphics[width=1.0\textwidth]{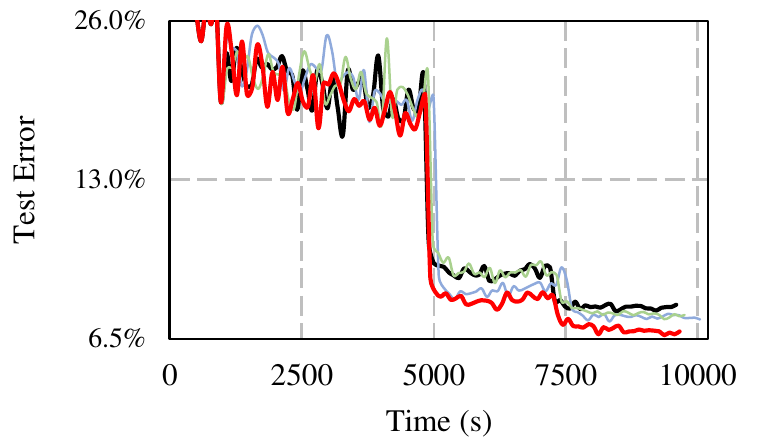}
\end{subfigure}
\hspace*{\spacebetweenfigs}
\begin{subfigure}[t!]{0.34\textwidth}
\includegraphics[width=1.0\textwidth]{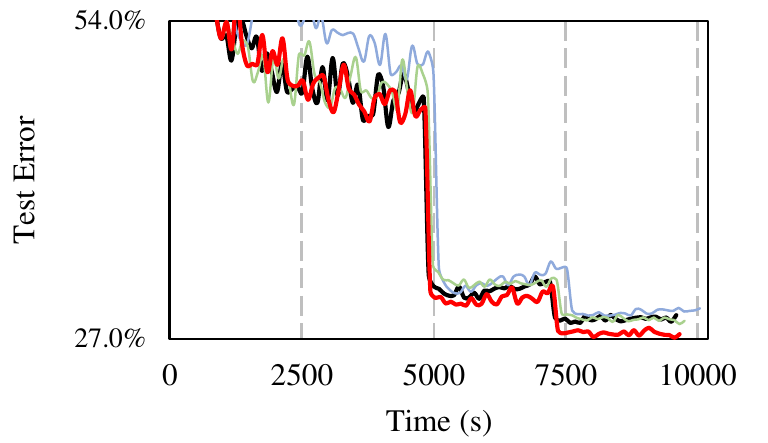}
\end{subfigure}
\end{center}
% requirement: 1 line spacing
\hspace*{1.25cm} {\small (d) MNIST Test Error.} \hspace*{3.00cm} {\small (e) CIFAR-10 Test Error.}  \hspace*{2.8cm}  {\small (f) CIFAR-100 Test Error.}
%%%%%%%%%%%%%%%%%% 2nd subfigure START %%%%%%%%%%%%%%%%%%
% requirement: 1 line spacing
\vspace*{-0.2cm}
\caption{Convergence curves of four batch selection methods using \enquote{DenseNet with momentum} (log-scale).}
\vspace*{-0.2cm}
\label{fig:densenet_momentum}
\end{figure*}
}

\begin{table*}[ht!]
%\small
\vspace*{+0.1cm}
\caption{The best test errors (\%) of four batch selection methods using {\bf DenseNet}.}
\vspace*{-0.3cm}
\begin{center}
\begin{tabular}{|L{2.2cm} ||X{2.0cm} |X{2.0cm} |X{2.0cm}|| X{2.0cm} |X{2.0cm} |X{2.0cm}|}\hline
Optimizer &  \multicolumn{3}{|c||}{Momentum in Figure \ref{fig:densenet_momentum}} & \multicolumn{3}{|c|}{SGD in Figure \ref{fig:densenet_sgd} (Appendix \ref{sec:generalization})}\\\hline
Method  &  \!\!MNIST\!\! &\!\!CIFAR-10\!\! & \!\!CIFAR-100\!\! & \!\!MNIST\!\! &\!\!CIFAR-10\!\! & \!\!CIFAR-100\!\! \\\hline\hline
\!{\randombatch{}}\!\!\! & \!\!0.53 $\pm$ 0.03\!\! & \!\!7.33 $\pm$ 0.09\!\! & \!\!27.95 $\pm$ 0.16\!\! & \!\!1.23 $\pm$ 0.03\!\! & \!\!14.86 $\pm$ 0.09\!\! & \!\!40.15 $\pm$ 0.06\!\! \\\hline
\!{\onlinebatch{}}\!\!\! & \!\!0.51 $\pm$ 0.01\!\! & \!\!7.00 $\pm$ 0.10\!\! & \!\!28.39 $\pm$ 0.25\!\! & \!\!0.77 $\pm$ 0.02\!\! & \!\!13.52 $\pm$ 0.02\!\! & \!\!40.72 $\pm$ 0.12\!\! \\\hline
\!{\activebias{}}\!\!\! &\!\!0.62 $\pm$ 0.03\!\! & \!\!7.07 $\pm$ 0.04\!\! & \!\!27.87 $\pm$ 0.11\!\! & \!\!{\bf 0.68} $\pm$ {\bf 0.02}\!\! & \!\!14.21 $\pm$ 0.25\!\! & \!\!42.87 $\pm$ 0.05\!\!  \\\hline
\!{\algname{}}\!\!\! & \!\!{\bf0.49} $\pm$ {\bf 0.02}\!\! & \!\!{\bf 6.60} $\pm$ {\bf 0.02}\!\! & \!\!{\bf 27.05} $\pm$ {\bf 0.19}\!\! & \!\!0.99 $\pm$ 0.06\!\! & \!\!{\bf 13.18} $\pm$ {\bf 0.11}\!\! & \!\!{\bf 38.65} $\pm$ {\bf 0.11}\!\! \\\hline
\end{tabular}
\end{center}
\label{table:densenet_result}
\end{table*}
\vspace*{-0.1cm}
\begin{table*}[ht!]
%\small
\vspace*{-0.1cm}
\caption{The best test errors (\%) of four batch selection methods using {\bf ResNet}.}
\vspace*{-0.3cm}
\begin{center}
\begin{tabular}{|L{2.2cm} ||X{2.0cm} |X{2.0cm} |X{2.0cm}|| X{2.0cm} |X{2.0cm} |X{2.0cm}|}\hline
Optimizer &  \multicolumn{3}{|c||}{Momentum in Figure \ref{fig:resnet_momentum} (Appendix \ref{sec:generalization})} & \multicolumn{3}{|c|}{SGD in Figure \ref{fig:resnet_sgd} (Appendix \ref{sec:generalization})}\\\hline
Method  &  \!\!MNIST\!\! &\!\!CIFAR-10\!\! & \!\!CIFAR-100\!\! & \!\!MNIST\!\! &\!\!CIFAR-10\!\! & \!\!CIFAR-100\!\! \\\hline\hline
\!{\randombatch{}}\!\!\! & \!\!0.64 $\pm$ 0.04\!\! & \!\!10.22 $\pm$ 0.12\!\! & \!\!33.20 $\pm$ 0.07\!\! & \!\!1.16 $\pm$ 0.03\!\! & \!\!12.73 $\pm$ 0.09\!\! & \!\!40.07 $\pm$ 0.16\!\!    \\\hline
\!{\onlinebatch{}}\!\!\! & \!\!0.67 $\pm$ 0.05\!\! & \!\!10.06 $\pm$ 0.05\!\! & \!\!33.38 $\pm$ 0.01\!\! & \!\!0.89 $\pm$  0.03\!\! & \!\!12.18 $\pm$ 0.08\!\! & \!\!40.69 $\pm$ 0.09\!\! \\\hline
\!{\activebias{}}\!\!\! & \!\!0.61 $\pm$ 0.04\!\! & \!\!10.55 $\pm$ 0.08\!\! & \!\!34.19 $\pm$ 0.07\!\! & \!\!{\bf0.80} $\pm$ {\bf0.01}\!\! & \!\!13.51 $\pm$ 0.07\!\! & \!\!45.62 $\pm$ 0.07\!\!  \\\hline
\!{\algname{}}\!\!\! & \!\!{\bf0.61} $\pm$ {\bf0.01}\!\! & \!\!{\bf9.79} $\pm$ {\bf0.04}\!\! & \!\!{\bf32.43} $\pm$ {\bf0.04}\!\! & \!\!0.97 $\pm$ 0.03\!\! & \!\!{\bf11.63} $\pm$ {\bf0.09}\!\! & \!\!{\bf38.94} $\pm$ {\bf0.14}\!\! \\\hline
\end{tabular}
\end{center}
\label{table:resnet_result}
\vspace*{-0.25cm}
\end{table*}

\section{Evaluation}
\label{sec:evaluation}

We empirically show the improvement of \algname{} over not only \randombatch{}\,(baseline) but also \onlinebatch{}\,\cite{loshchilov2015online} and \activebias{}\,\cite{chang2017active}, which are two state-of-the-art adaptive batch selections. In particular, we elaborate on the effect of the sliding window approach\,(\algname{}) compared with the growing window approach\,(\activebias{}).
\randombatch{} selects next mini-batch samples uniformly at random from the entire dataset. 
\onlinebatch{} selects hard samples based on the rank of the loss computed from previous epochs. \activebias{} selects uncertain samples with high variance of true label probabilities in the growing window.
All the algorithms were implemented using TensorFlow $1.15$ and executed using a single NVIDIA Titan Volta GPU. For reproducibility, we provide the source code at \url{https://github.com/kaist-dmlab/RecencyBias}. 

Image classification and fine-tuning tasks were performed to validate the superiority of \algname{}.
Because fine-tuning is used to quickly adapt to a new dataset, it is suitable to reap the benefit of fast training speed.
In support of reliable evaluation, we repeated every task \emph{thrice} and reported the average and standard error of the best test errors. The \emph{best test error} in a given time has been widely used for the studies on fast and accurate training\,\cite{katharopoulos2018not,loshchilov2015online}. All the datasets used for experiments are summarized in Table \ref{table:datasets}.

{
\newcommand{\spacebetweenfigs}{+0.45cm}
\newcommand{\spacebeforecaption}{0.1cm}

\begin{figure*}[t]
% make more space by increasing margin <- may not need this
\advance\leftskip-0.3cm % increase left margin temporarily
% \advance\rightskip-1.0cm % increase right margin temporarily
\begin{subfigure}[t!]{1.0\textwidth}
\begin{center}
\includegraphics[width=0.7\textwidth]{figure/experiments/label}
\vspace*{0.1cm}
\end{center}
\end{subfigure}
%%%%%%%%%%%%%%%%%% 1st subfigure START %%%%%%%%%%%%%%%%%%
\begin{center}
\begin{subfigure}[t!]{0.38\textwidth}
\includegraphics[width=1.0\textwidth]{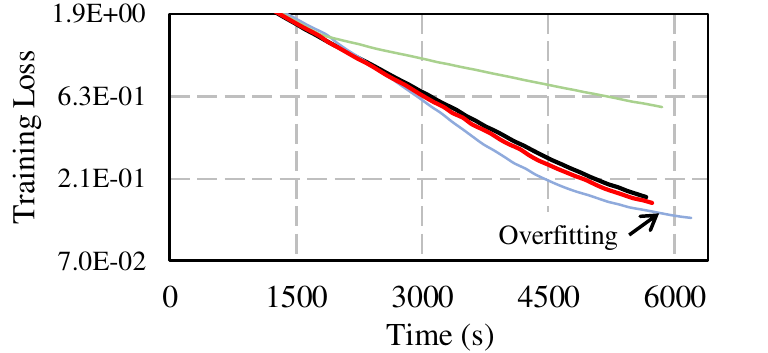}
\vspace*{\spacebeforecaption}
\end{subfigure}
\hspace*{\spacebetweenfigs}
\begin{subfigure}[t!]{0.38\textwidth}
\includegraphics[width=1.0\textwidth]{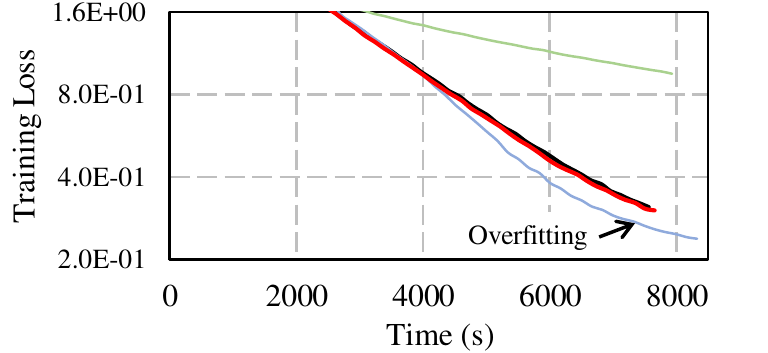}
\vspace*{\spacebeforecaption}
\end{subfigure}
\end{center}
\vspace*{-0.4cm}
\hspace*{1.26cm} {\small (a) MIT-67 Training Loss.} \hspace*{4.17cm} {\small (b) Food-100 Training Loss.}
\vspace*{0.1cm}
\begin{center}
\begin{subfigure}[t!]{0.38\textwidth}
\includegraphics[width=1.0\textwidth]{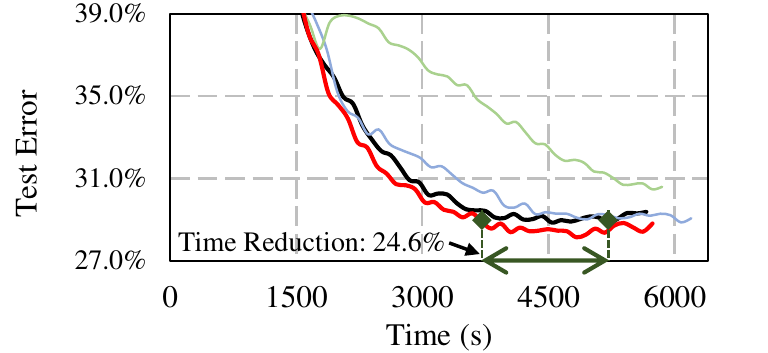}
\vspace*{\spacebeforecaption}
\end{subfigure}
\hspace*{\spacebetweenfigs}
\begin{subfigure}[t!]{0.38\textwidth}
\includegraphics[width=1.0\textwidth]{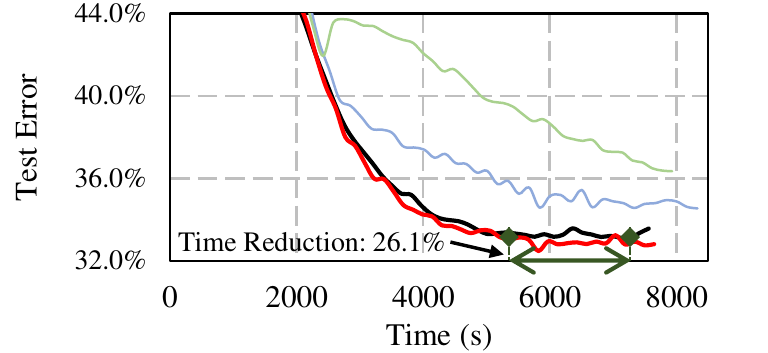}
\vspace*{\spacebeforecaption}
\end{subfigure}
\end{center}
\vspace*{-0.4cm}
\hspace*{1.28cm} {\small (c) MIT-67 Test Error} \hspace*{4.67cm} {\small (d) Food-100 Test Error.}
\vspace*{-0.3cm}
\caption{Convergence curves for fine-tuning on two benchmark datasets.}
\label{fig:fine_tune}
\vspace*{-0.3cm}
\end{figure*}
}

\subsection{Analysis on Selected Mini-batch Samples}
\label{sec:analysis_sample}

For an in-depth analysis on selected samples, we plot the loss distribution of mini-batch samples selected from CIFAR-10 by four different methods in Figure \ref{fig:loss_distribution}.
\emph{(i)} The distribution of \onlinebatch{} is the most skewed toward high loss by the design principle of selecting hard samples. 
\emph{(ii)} \activebias{} emphasizes \emph{moderately hard} samples at an early training stage in considering that its loss distribution lies between those of \emph{Random Batch} and \emph{Online Batch}. However, owing to the outdated predictions caused by the growing window, the proportion of easy samples with low loss increases at a late training stage. These easy samples, which are misclassified as uncertain at that stage, tend to make the convergence of training slow down. 
\emph{(iii)} In contrast to \activebias{}, by virtue of the sliding window, the distribution of \algname{} lies between those of \randombatch{} and \onlinebatch{} \textrm{regardless of the training stage}. Consequently, \algname{} continues to highlight the moderately hard samples, which are likely to be informative, during the training process. \looseness=-1

\subsection{Task I: Image Classification}
\label{sec:eval_classification}

\subsubsection{Experiment Setting}
We trained DenseNet (L=40, k=12) and ResNet (L=50) with a momentum optimizer and an SGD optimizer on \emph{three} benchmark datasets: MNIST\,(10 classes)\footnote{\url{http://yann.lecun.com/exdb/mnist}}, classification of handwritten digits\,\cite{lecun1998mnist}, and CIFAR-10\,($10$ classes)\footnote{\url{https://www.cs.toronto.edu/~kriz/cifar.html}} and CIFAR-100\,($100$ classes)\footnotemark[3], classification of a subset of $80$ million categorical images\,\cite{krizhevsky2014cifar}.
For the classification task, we used data augmentation, batch normalization, a dropout of $0.2$, a momentum of $0.9$, and a batch size of $128$. 
Regarding the algorithm parameters, we fixed the window size $q=10$ and the initial selection pressure $s_{e_0}=100$,\footnote{\onlinebatch{} also used the same decaying strategy.} which were the best values found by the grid search in Section \ref{sec:param_sel}. The warm-up epoch $\gamma$ was set to be $10$. To reduce the performance variance caused by randomly initialized model parameters, all parameters were shared by all methods during the warm-up period.
Regarding the training schedule, we trained the network for $40,000$ iterations and used an initial learning rate of $0.1$, which was divided by $10$ at $50\%$ and $75\%$ of the total number of training iterations. 

\subsubsection{Results}
Figure \ref{fig:densenet_momentum} shows the convergence curves of training loss and test error for four batch selection methods using DenseNet and a momentum optimizer. 
In order to highlight the improvement of \algname{} over the baseline\,(\randombatch{}), their lines are dark colored.
The best test errors in Figures \ref{fig:densenet_momentum}(d), \ref{fig:densenet_momentum}(e), and \ref{fig:densenet_momentum}(f) are summarized on the left side of Table \ref{table:densenet_result}. \looseness=-1

In general, \algname{} achieved the most accurate network while accelerating the training process on all datasets.
The training loss of \algname{} converged faster (Figures \ref{fig:densenet_momentum}(a), \ref{fig:densenet_momentum}(b), and \ref{fig:densenet_momentum}(c)) without the increase in the generalization error, thereby achieving the lowest test error (Figures \ref{fig:densenet_momentum}(d), \ref{fig:densenet_momentum}(e), and \ref{fig:densenet_momentum}(f)). In contrast, the test error of \onlinebatch{} was not the best even if its training loss converged the fastest among all methods. As the training difficulty increased from CIFAR-10 to CIFAR-100, the test error of \onlinebatch{} became even worse than that of \randombatch{}. 
That is, emphasizing hard training samples rather worsened the generalization capability of the network for the hard dataset\,(i.e., CIFAR-100) because of the overfitting. 
On the other hand, \algname{} expedited the training step as well as achieved the lowest test error even for the hard dataset.
Meanwhile, \activebias{} was prone to make the network better generalized on test data. In CIFAR-10, despite its highest training loss, the test error of \activebias{} was better than that of \randombatch{}. However, \activebias{} slowed down the training process because of the limitation of growing windows, as discussed in Section \ref{sec:analysis_sample}.
We note that, although both \algname{} and \activebias{} exploited uncertain samples, only \algname{} based on sliding windows succeeded to not only speed up the training process but also reduce the generalization error. 
Quantitatively, \algname{} achieved a significant reduction in test error of $3.22$\%--$9.96\%$ compared with \randombatch{}, $4.72\%$--$5.71\%$ compared with \onlinebatch{}, and $2.94\%$--$20.97\%$ compared with \activebias{}.

\begin{table}[t!]
\caption{The best test errors (\%) of four batch selection strategies using DenseNet in Figure \ref{fig:fine_tune}.}
\vspace*{-0.3cm}
\begin{center}
\begin{tabular}{|L{2.2cm} ||X{2.3cm} |X{2.3cm}|}\hline
Method  &  \!\!MIT-67\!\! &\!\!Food-100\!\! \\\hline\hline
\!{\randombatch{}}\!\!\! & \!\!28.85 $\pm$ 0.26\!\! & \!\!33.07 $\pm$ 0.60\!\! \\\hline
\!{\onlinebatch{}}\!\!\! & \!\!28.87 $\pm$ 0.55\!\! & \!\!34.54 $\pm$ 0.24\!\! \\\hline
\!{\activebias{}}\!\!\! &\!\!30.46 $\pm$ 0.31\!\! & \!\!36.35 $\pm$ 0.47\!\!  \\\hline
\!{\algname{}}\!\!\! & \!\!{\bf 28.02} $\pm$ {\bf 0.32}\!\! & \!\!{\bf 32.47} $\pm$ {\bf 0.61}\!\! \\\hline
\end{tabular}
\end{center}
\label{table:finetune}
\vspace*{-0.2cm}
\end{table}

\begin{table*}[t!]
%\small
\vspace*{-0.20cm}
\caption{\textbf{Recency~Bias}'s reduction in training time over other batch selection methods.}
\vspace*{-0.3cm}
\begin{center}
\begin{tabular}{|L{2.2cm} ||X{6cm} |X{6cm}|}\hline
Method  &  MIT-67 & FOOD-100 \\ \hline\hline
\!{\randombatch{}}\!\!\! & \!\!$(5,218-3,936)/5,218\times100={\bf 24.57\%}$\!\! & \!\!$(7,263-5,365)/7,263\times100={\bf 26.13}\%$\!\! \\\hline
\!{\onlinebatch{}}\!\!\! & \!\!$(6,079-3,823)/6,079\times100={\bf 37.11\%}$\!\! & \!\!$(8,333-3,685)/8,333\times100={\bf 55.78}\%$\!\! \\\hline
\!{\activebias{}}\!\!\! & \!\!$(5,738-3,032)/5,738\times100={\bf 47.16\%}$\!\! & \!\!$(7,933-3,227)/7,933\times100={\bf 59.32\%}$\!\!  \\\hline
\end{tabular}
\end{center}
\label{table:time_reduction}
\vspace*{-0.05cm}
\end{table*}

\begin{table*}[ht!]
\begin{center}
\parbox{.5\textwidth}{%
\parbox{8.4cm}{
\caption{The converged training loss of \algname{} with the four strategies of decaying the selection pressure.}
\vspace*{-0.2cm} \label{table:ablation_training_loss}
}
\begin{tabular}{|L{3.4cm} || X{1.95cm}| X{1.95cm} |}\hline
Metric &  \multicolumn{2}{|c|}{Training Loss} \\\hline
Decaying Strategy  &  \!\!CIFAR-10\!\! & \!\!CIFAR-100\!\! \\\hline\hline
Strategy 1 ($s_e=10$) & $0.0059$ & $0.2172 $  \\\hline
Strategy 2 ($s_e=100$) & $0.0036$  & $0.1757 $ \\\hline
Strategy 3 ($s_e=10 \rightarrow 1$)\!\! & $0.0120 $ & $0.2539$ \\\hline
Strategy 4 ($s_e=100 \rightarrow 1$)\!\! & $0.0060 $ & $0.2207$ \\\hline
\end{tabular}
}%
\hspace*{0.4cm}
\parbox{.5\textwidth}{%
\parbox{8.4cm}{
\caption{The best test error\,(\%) of \algname{} with the four strategies of decaying the selection pressure.}
\vspace*{-0.2cm}
\label{table:ablation_test_error}
}
\begin{tabular}{|L{3.4cm} || X{1.95cm}| X{1.95cm} |}\hline
Metric &  \multicolumn{2}{|c|}{Test Error}\\\hline
Decaying Strategy  & \!\!CIFAR-10\!\! & \!\!CIFAR-100\!\! \\\hline\hline
Strategy 1 ($s_e=10$) & $6.81 \pm 0.02$ & $27.84 \pm 0.17$ \\\hline
Strategy 2 ($s_e=100$) & $6.62 \pm 0.04$ & $28.19 \pm 0.15$ \\\hline
Strategy 3 ($s_e=10 \rightarrow 1$)\!\! & $7.13 \pm 0.03$ & $27.47 \pm 0.22$ \\\hline
Strategy 4 ($s_e=100 \rightarrow 1$)\!\! & $6.60 \pm 0.02$ & $27.05 \pm 0.19$  \\\hline
\end{tabular}
}%
\end{center}
\vspace*{-0.1cm}
\end{table*}

The results of the best test error for ResNet or an SGD optimizer are summarized in Tables \ref{table:densenet_result} and \ref{table:resnet_result} (see Appendix \ref{sec:generalization} for more details). Regardless of a neural network and an optimizer, \algname{} achieved the lowest test error except in MNIST with an SGD optimizer.
The improvement of \algname{} over the others was higher with an SGD optimizer than with a momentum optimizer. \looseness=-1

\vspace*{-0.05cm}
\subsection{Task II: Fine-Tuning}
\label{sec:eval_fine_tuning}
\subsubsection{Experiment Setting} 
Our second experiment was fine-tuning a pretrained network on new datasets. Because there exist many powerful pretrained networks on large datasets, this task is an important application to verify the usefulness of fast training.
We prepared DenseNet\,(L=121, k=32) previously trained on ImageNet\,\cite{deng2009imagenet} and then fine-tuned the network on \emph{two} benchmark datasets: MIT-67\,($67$ classes)\footnote{\url{http://web.mit.edu/torralba/www/indoor.html}}, classification of indoor scenes\,\cite{quattoni2009recognizing}, and Food-100 ($100$ classes)\footnote{\url{http://foodcam.mobi/dataset100.html}}, classification of popular foods in Japan\,\cite{kawano14b}. As summarized in Table \ref{table:datasets}, all training and testing images in both datasets were resized to $256\times256$, which is the original input size of the pretrained network.
For the fine-tuning task, the network was trained end-to-end for $50$ epochs with a batch size $32$ and a constant learning rate $2\times10^{-4}$ after replacing the last classification layer. The other configurations were the same as those in Section \ref{sec:eval_classification}. \looseness=-1

\subsubsection{Results on Test Error}
Figure \ref{fig:fine_tune} shows the convergence curves of training loss and test error for the fine-tuning task on MIT-67 and Food-100. The best test errors in Figure \ref{fig:fine_tune} are summarized in Table \ref{table:finetune}. Overall, all convergence curves showed similar trends to those of the classification task in Figure \ref{fig:densenet_momentum}. Only \algname{} converged faster than \randombatch{} in both training loss and test error.
\onlinebatch{} converged the fastest in training loss, but its test error was rather higher than \randombatch{} owing to the overfitting. \activebias{} converged the slowest in both training loss and test error. Quantitatively, compared with \randombatch{}, \algname{} reduced the test error by $2.88\%$ and $1.81\%$ in MIT-67 and Food-100, respectively. \looseness=-1

\subsubsection{Results on Training Time}
Moreover, to assess the performance gain in training time, we computed the reduction in the training time taken to reach the same error. For example, in Figure\,\ref{fig:fine_tune}(c), the best test error of $28.85\%$ achieved in $5,218$ seconds by \randombatch{} could be achieved only in $3,936$ seconds by \algname{}; thus, \algname{} improved the training time by $24.6\%$.
Table \ref{table:time_reduction} summarizes the reduction in the training time of \algname{} over three other batch selection methods. Notably, \algname{} improved the training time by $24.57\%$--$47.16\%$ and $26.13\%$--$59.32\%$ in fine-tuning MIT-67 and FOOD-100 datasets, respectively.

\subsection{Ablation  Study on Selection Pressure}
\label{sec:abalation}

To examine the effect of decaying the selection pressure, we conducted additional ablation experiments on \emph{four} different strategies of decaying the $s_e$ value, as follows:
\vspace*{-0.0cm}
\begin{enumerate}[label={\arabic*.}, leftmargin=12pt] % , noitemsep
\item \textbf{$\boldsymbol{s_e= 10}$}: A \emph{wide range} of uncertain training samples were selected during the \emph{entire} training process because $s_e$ was set to be a small constant value.
\vspace*{0.05cm}
\item \textbf{$\boldsymbol{s_e= 100}$}: Differently to Strategy 1, only \emph{highly} uncertain training samples were highlighted during the \emph{entire} training process because $s_e$ was set to be a large constant value. 
\vspace*{0.05cm}
\item \textbf{$\boldsymbol{s_e= 10 \rightarrow 1}$}: This strategy starts from Strategy 1, but, as the training progresses, more diverse training samples were chosen regardless of their uncertainty due to the exponential decay of  $s_e$ from $10$ to $1$.
\vspace*{0.05cm}
\item \textbf{$\boldsymbol{s_e= 100 \rightarrow 1}$}: This strategy is similar to Strategy 3, but the initial value of $s_e$ was made much larger to further emphasize highly uncertain samples during the training process.
\end{enumerate}

DenseNet (L=40, k=12) was trained on two CIFAR datasets using \algname{} with these four strategies. A momentum optimizer was used, and the other configurations were the same as those in Section \ref{sec:eval_classification}. Tables \ref{table:ablation_training_loss} and \ref{table:ablation_test_error} summarize the converged training loss and the best test error, respectively, of the four strategies.

A lower training loss was generally achieved when the selection pressure was not decayed (see Strategy 1 and Strategy 2 in Table \ref{table:ablation_training_loss}) because moderately hard training samples were consistently emphasized during the entire training process. Nevertheless, because using only a \emph{part} of training data exacerbated the overfitting problem as mentioned in Section \ref{sec:sample_prob}, the test error of Strategy 1\,(or Strategy 2) was generally worse than that of Strategy 3\,(or Strategy 4) as shown in Table \ref{table:ablation_test_error}. The overfitting problem may become severer with Strategy 2 than Strategy 1 as witnessed by the test error for CIFAR-100. As opposed to when using a constant selection pressure, a larger initial value was beneficial when decaying the selection pressure. That is, Strategy 4 achieved the lower test error than Strategy 3 in both datasets. Overall, the best test error was always achieved by Strategy 4, which corroborates the importance of decaying the selection pressure to exploit more diverse training samples especially at a late stage of training. \looseness=-1

\begin{figure}[t!]
\begin{center}
\includegraphics[width=8.6cm]{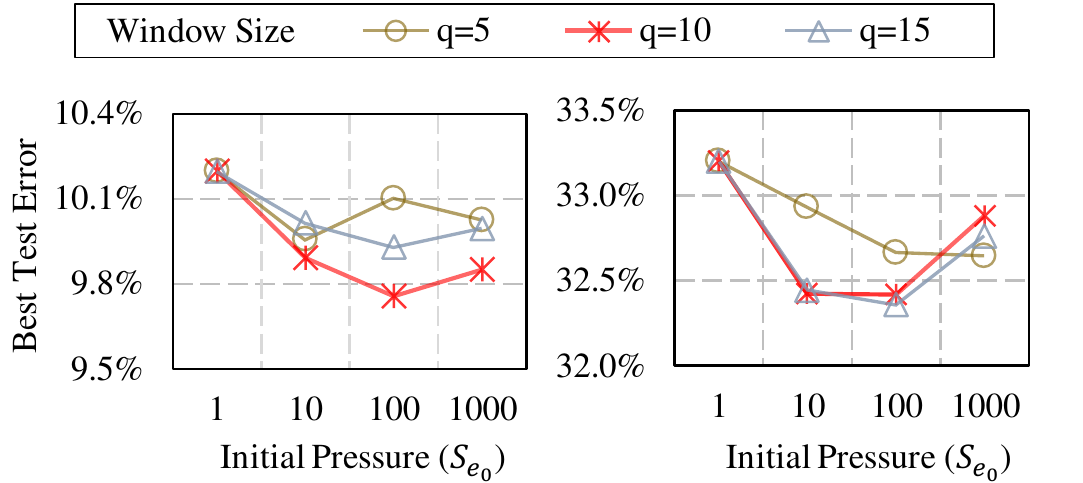}
\end{center}
\vspace*{-0.075cm}
\hspace*{1.2cm} {\small (a) CIFAR-10.}  \hspace*{2.25cm} {\small (b) CIFAR-100.}  
\vspace*{-0.275cm}
\caption{Grid search on CIFAR-10 and CIFAR-100 datasets using ResNet.}
\label{fig:grid_search}
\vspace*{-0.175cm}
\end{figure}

\subsection{Hyperparameter Selection}
\label{sec:param_sel}

\algname{} receives the two hyperparameters: \emph{(i)} the initial selection pressure $s_{e_0}$ that determines the sampling probability gap between the most and the least uncertain samples and \emph{(ii)} the window size $q$ that determines how many recent label predictions are involved in predicting the uncertainty. To decide the best hyperparameters, we trained ResNet\,(L=50) on CIFAR-10 and CIFAR-100 with a momentum optimizer. For hyperparameters selection, the two hyperparameters were chosen in a grid $s_{e_0} \in \{1, 10, 100, 1000\}$ and $q \in \{5, 10, 15\}$.

Figure \ref{fig:grid_search} shows the test errors of \algname{} obtained by the grid search on the two datasets. Regarding the initial selection pressure $s_{e_0}$, the lowest test error was typically achieved when the $s_{e_0}$ value was $100$. As for the window size $q$, the test error was almost always the lowest when the $q$ value was $10$. Similar trends were observed for the other combinations of a neural network and an optimizer. Therefore, in all experiments, we set $s_{e_0}$ to be $100$ and $q$ to be $10$.

%% file: 6-conclusion.tex
\section{Conclusion}
\label{sec:conclusion}
In this paper, we presented a novel adaptive batch selection algorithm called \algname{} that emphasizes predictively uncertain samples for accelerating the training of neural networks. Toward this goal, the predictive uncertainty of each sample is evaluated using its \emph{recent} label predictions managed by a sliding window of a fixed size. Then, uncertain samples \emph{at the moment} are selected with high probability for the next mini-batch. 
We conducted extensive experiments on both classification and fine-tuning tasks. The results showed that \algname{} is effective in reducing the training time as well as the best test error. It was worthwhile to note that using \emph{all} historical observations to estimate the uncertainty has the side effect of slowing down the training process. Overall, a merger of uncertain samples and sliding windows greatly improves the power of adaptive batch selection.

%% file: 8-appendix.tex
\vspace*{0.7cm}
\newcommand{\spacebetweenfigs}{-0.45cm}
\newcommand{\spacebeforecaption}{0.2cm}
\noindent
\begin{minipage}[c]{\textwidth}
\centering
\includegraphics[width=0.7\textwidth]{figure/experiments/label} 
\includegraphics[width=0.34\textwidth]{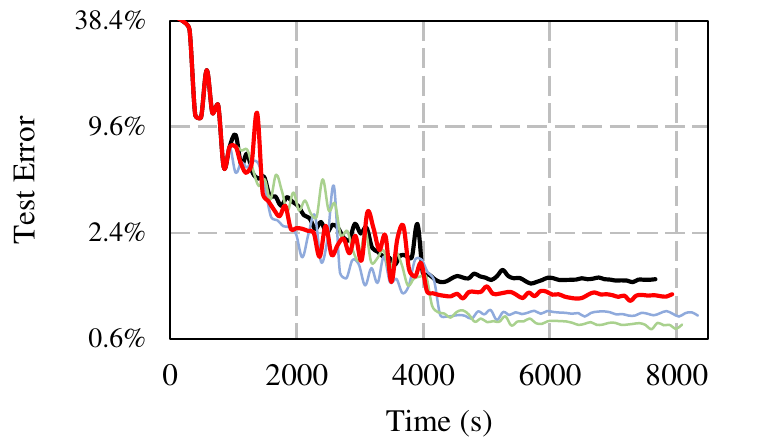}
\hspace*{\spacebetweenfigs}
\includegraphics[width=0.34\textwidth]{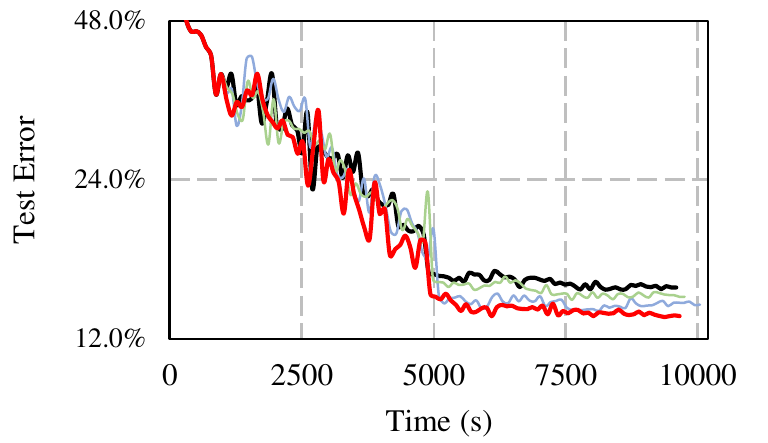}
\hspace*{\spacebetweenfigs}
\includegraphics[width=0.34\textwidth]{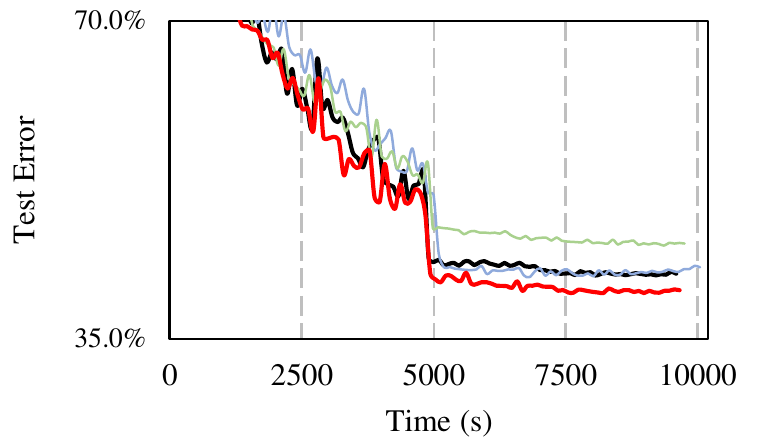}
\hspace*{0.95cm} {\small (a) MNIST Test Error.} \hspace*{2.95cm} {\small (b) CIFAR-10 Test Error.}  \hspace*{2.8cm}  {\small (c) CIFAR-100 Test Error.}
\vspace*{-0.25cm}
\captionof{figure}{Convergence curves of four batch selection strategies using \enquote{DenseNet with SGD} (log-scale).}
\vspace*{0.25cm}
\label{fig:densenet_sgd}
\end{minipage}
\begin{minipage}[c]{\textwidth}
\centering
\includegraphics[width=0.34\textwidth]{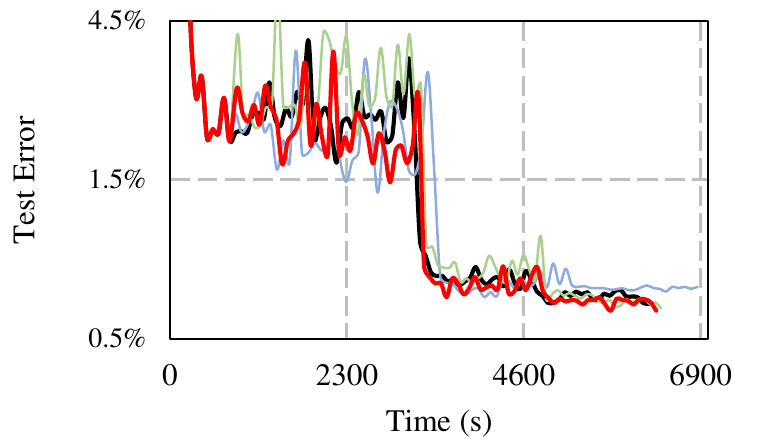}
\hspace*{\spacebetweenfigs}
\includegraphics[width=0.34\textwidth]{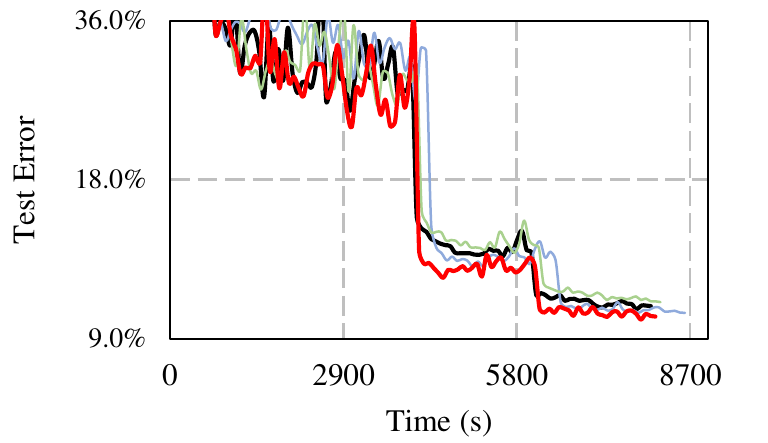}
\hspace*{\spacebetweenfigs}
\includegraphics[width=0.34\textwidth]{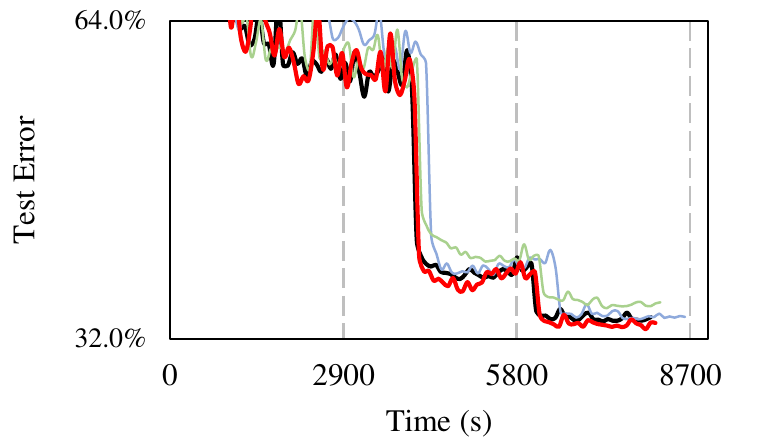}
\hspace*{0.95cm} {\small (a) MNIST Test Error.} \hspace*{2.95cm} {\small (b) CIFAR-10 Test Error.}  \hspace*{2.8cm}  {\small (c) CIFAR-100 Test Error.}
\vspace*{-0.25cm}
\captionof{figure}{Convergence curves of four batch selection strategies using \enquote{ResNet with momentum} (log-scale).}
\vspace*{0.25cm}
\label{fig:resnet_momentum}
\end{minipage}
\begin{minipage}[c]{\textwidth}
\centering
\includegraphics[width=0.34\textwidth]{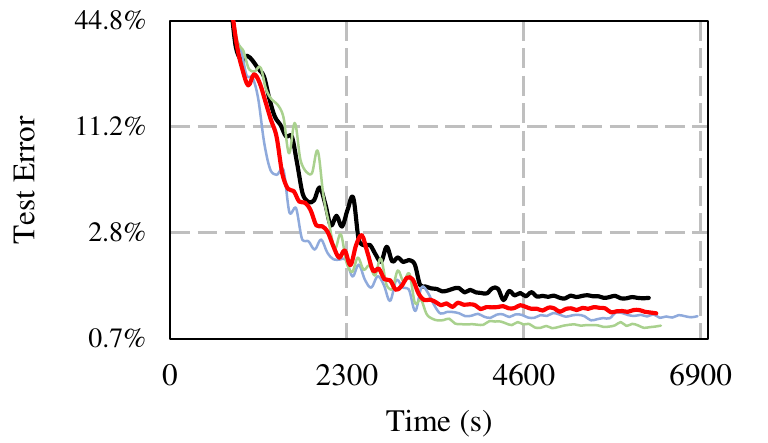}
\hspace*{\spacebetweenfigs}
\includegraphics[width=0.34\textwidth]{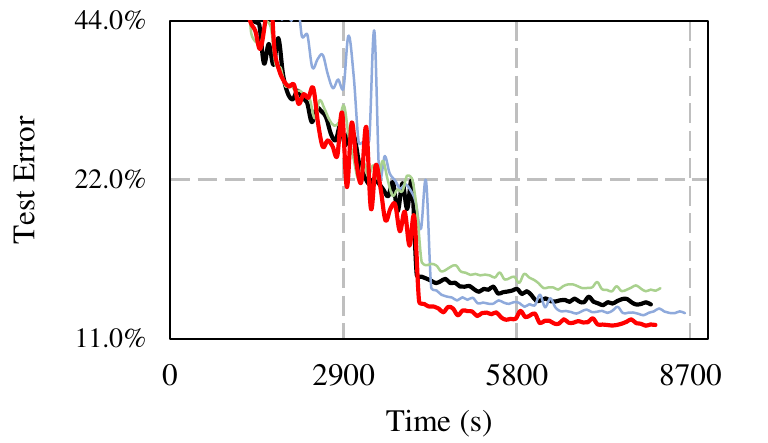}
\hspace*{\spacebetweenfigs}
\includegraphics[width=0.34\textwidth]{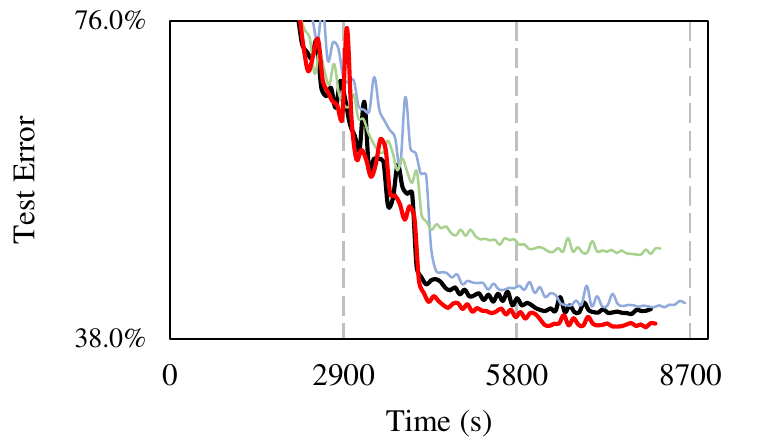}
\hspace*{0.95cm} {\small (a) MNIST Test Error.} \hspace*{2.95cm} {\small (b) CIFAR-10 Test Error.}  \hspace*{2.8cm}  {\small (c) CIFAR-100 Test Error.}
\vspace*{-0.25cm}
\captionof{figure}{Convergence curves of four batch selection strategies using \enquote{ResNet with SGD} (log-scale).}
\label{fig:resnet_sgd}
\end{minipage}

\section{Generalizability of \algname{}}
\label{sec:generalization}
\vspace*{0.1cm}

Figures \ref{fig:densenet_sgd}, \ref{fig:resnet_momentum}, and \ref{fig:resnet_sgd} show the convergence curves of test error for the four batch selection strategies using \enquote{DenseNet and an SGD optimizer}\,(see the right side of Table \ref{table:densenet_result}), \enquote{ResNet and a momentum optimizer}\,(see the left side of Table \ref{table:resnet_result}), and \enquote{ResNet and an SGD optimizer}\,(see the right side of Table \ref{table:resnet_result}), respectively. 

The performance dominance of \algname{} was generally consistent regardless of an optimizer and a network architecture. Except in MNIST with an SGD optimizer\,(Figure \ref{fig:densenet_sgd}(a) and Figure \ref{fig:resnet_sgd}(a)), \algname{} achieved a significant reduction in test error of $2.32\%$--$19.51\%$ compared with 
\randombatch{}, $2.51\%$--$8.96\%$ compared with \onlinebatch{}, and $5.15\%$--$14.64\%$ compared with \activebias{}, respectively.
Therefore, based on this confirmed generalizability, we expect that \algname{} can be smoothly applied for a new emerging optimizer and/or network architecture.